\documentclass{amsart}

\usepackage{latexsym, amssymb, amsfonts, amsmath, amsthm, graphics, graphicx, subfigure, bbm, stmaryrd}

\usepackage{graphicx}
\usepackage{verbatim}

\newtheorem{theorem}{Theorem}[section]
\newtheorem{lemma}[theorem]{Lemma}
\newtheorem{corollary}[theorem]{Corollary}

\theoremstyle{definition}

\theoremstyle{remark}

\numberwithin{equation}{section}

\begin{document}

\bibliographystyle{plain}

\title{Convergence Rate of Krasulina Estimator}

\author{Jiangning Chen}
\address{School of Mathematics, Georgia Institute of Technology, Altanta, GA 30313}

\email{jchen444@math.gatech.edu}


\date{Nov.14 2017}

\keywords{PCA, incremental, online updating, covariance matrix, eigenvector, rate of convergence, adaptive estimation}

\begin{abstract}

Principal component analysis (PCA) is one of the most commonly used statistical procedures with a wide range of applications. Consider the points $X_1, X_2,..., X_n$ are vectors drawn i.i.d. from a distribution with mean zero and covariance $\Sigma$, where $\Sigma$ is unknown. Let $A_n = X_nX_n^T$, then $E[A_n] = \Sigma$. This paper considers the problem of finding the smallest eigenvalue and eigenvector of matrix $\Sigma$. A classical estimator of this type is due to Krasulina\cite{krasulina_method_1969}. We are going to state the convergence proof of Krasulina for the smallest eigenvalue and corresponding eigenvector, and then find their convergence rate.
\end{abstract}

\maketitle

\section{Introduction}

Principal component analysis (PCA) is one of the most widely used dimension reduction techniques in data analysis. Suppose $X_1, X_2,..., X_n$ are vectors drawn i.i.d. from a distribution with mean zero and covariance $\Sigma$, where $\Sigma\in\mathbb{R}^{d\times d}$ is unknown. Let $A_n = X_nX_n^T$, then $E[A_n] = \Sigma$. We are interested in finding eigenvalues of matrix $\Sigma$ and the corresponding eigenvectors if identifiable.

This problem has been intensively studied especially in the offline setting where all the observations are available at once, see \cite{arora_stochastic_2013, blanchard_statistical_2007, cai_sparse_2013, roweis_em_1998, vu_minimax_2012, warmuth_randomized_2007, zwald_convergence_2006}. For instance, \cite{cai_sparse_2013} derived the sharp minimax rate of estimation of the eigenvectors for the following Frobenius risk $E[\|\Theta\Theta^T-\hat{\Theta}\hat{\Theta}^T\|_F^2]$, where $\Theta = [\theta_1,\theta_2,...,\theta_r]$ is the matrix of eigenvectors and $\hat{\Theta}$ is the corresponding estimator. Recently, \cite{koltchinskii_concentration_2014,koltchinskii2016asymptotics,koltchinskii2017normal} derived subtle results about the behavior of the standard PCA method in an infinite-dimensional setting.

In the high dimensional setting and for massive data sets, the computational complexity of PCA may become an issue. Indeed, for data in $\mathbb{R}^d$, the default method needs storage space in the order of $O(d^2)$. Therefore, it is interesting to develop online incremental schemes that only take one data point at a time to update estimators of eigenvectors and eigenvalues. The least storage consuming methods only need $O(d)$ space to compute one eigenvector. 

Assume matrix $\Sigma$ has the standard decomposition:
\begin{equation}\label{AAssumptino}
\Sigma = \sum_{j = 1}^d \lambda_j \theta_j\otimes\theta_j,
\end{equation}
where eigenvalues $\lambda_j$'s satisfy: $\lambda_1<\lambda_2\leq \lambda_3
\leq...<\lambda_d$ and $\theta_j$ are the corresponding eigenvectors. We assume here that $\lambda_1 < \lambda_2$ so that $\theta_1$ is identifiable up to sign. 
To compute the smallest eigenvalue and corresponding eigenvector, Krasulina\cite{krasulina_method_1969} suggested the following stochastic gradient scheme. At time $n+1$, the estimate of the smallest eigenvector $V_{n+1}$ is updated as follows:
\begin{equation}\label{Krasulina1}
V_{n+1} = V_n - \gamma_{n+1}\xi_{n+1},
\end{equation}
where $\{\gamma_{n}\}$ is the learning rate, typically, $\{\gamma_n\}$ is chosen such that 
\begin{equation}\label{GammaAssumption}
\sum \gamma_n = \infty\text{, }\sum \gamma_n^2 < \infty.
\end{equation}
For example, $\gamma_n = \frac{c}{n}$ where $c$ is an absolute constant, in practice, we can choose $c = 1$. And 
\begin{align*}
 \xi_{n+1} &= <X_{n+1},V_n>\cdot X_{n+1} - \frac{<X_{n+1},V_n>^2}{\|V_n\|^2}\cdot V_n\\
 &= A_{n+1}\cdot V_n - \frac{<A_{n+1} V_n,V_n>}{\|V_n\|^2}\cdot V_n.
\end{align*}

There has been a lot of effort to compute the spectrum decomposition. Oja and Karhunen\cite{oja_stochastic_1985} suggested a method which is closely related to Krasulina's, they use the update for the leading eigenvector as follows:
\begin{equation}\label{Oja}
V_{n+1} = \frac{V_n+\gamma_{n+1}<X_{n+1},V_n>X_{n+1}}{\|V_n+\gamma_{n+1}<X_{n+1},V_n>X_{n+1}\|}.
\end{equation}

\cite{krasulina_method_1969,oja_stochastic_1985} proved that these estimators converge almost surely under the assumption \eqref{AAssumptino}, \eqref{GammaAssumption} and $E[\|X_n\|^k]<\infty$ for some suitable $k$.

There are many other incremental estimators whose convergence has not been established yet. \cite{weng_candid_2003} introduces a candid covariance-free incremental PCA algorithm with assumption \eqref{AAssumptino}, they suggest the estimator: \begin{equation}\label{weng2003}
V_{n+1}=\frac{n-1-l}{n}V_{n-1}+\frac{1+l}{n}X_nX_n^T\frac{V_{n-1}}{\|V_{n-1}\|},
\end{equation}
where $l$ is called the amnesic parameter. With the presence of $l$, larger weight is given to new “samples” and the effect of old “samples” will fade out gradually. Typically, $l$ ranges from $2$ to $4$. They also addressed the estimation of additional eigenvectors by first subtracting from the data its projection on the estimated eigenvectors, then applying $\eqref{weng2003}$. \cite{arora_stochastic_2012} considers PCA problem as stochastic optimization problem, it considers an unknown source distribution over $\mathbb{R}^d$, and would like to find the k-dimensional subspace maximizing the variance of the distribution inside the subspace. They solve the problem by stochastic gradient descent, and suggests the updates:
$$V_{n+1} = \mathcal{P}_{orth}(V_n + \eta_nX_nX_n^TV_n),$$ where $\mathcal{P}_{orth}(V)$ performs a projection with respect to the spectral norm of $VV^T$ onto the
set of $d\times d$ matrices with $k$ eigenvalues equal to $1$ and the rest $0$, $\eta_n$ is the step size.

There also exist many results which analyze incremental PCA from the statistical perspective. They mainly show the asymptotic consistency of estimators under certain conditions. For example, \cite{mitliagkas_memory_2013} suggests a Block-Stochastic Power Method. \cite{jain_streaming_2016} finds an upper bound in probability $1-\delta$ of alignment loss function $1-\frac{<V_n,\theta_1>^2}{\|V_n\|^2}$ for Oja's estimator.

As for non-asymptotic result, \cite{balsubramani_fast_2013} derives sub-optimal bound on the alignment loss
$
L(V_n,\theta_1) := E\left[1-\frac{<V_n,\theta_1>^2}{\|V_n\|^2}\right],
$ and \cite{de_sa_accelerated_2017} introduces Mini-batch Power Method.

\medskip

Krasulina states the convergence of the smallest eigenvalue and eigenvector estimators, but did not provide convergence rate. In this paper, we find the rate of convergence for both eigenvalue and eigenvector estimators of Krasulina \eqref{Krasulina1} under a relatively mild assumption. Our analysis reveals a slower rate of convergence of eigenvalue estimator $\hat{\lambda_1} = \frac{<A_nV_n,V_n>}{\|V_n\|^2}$ and corresponding eigenvector estimator $\hat{\theta_1} = \frac{V_n}{\|V_n\|}$ as compared to the offline setting for Krasulina's scheme.

\medskip

\noindent\textbf{Notations:} for any vector $x\in\mathbb{R}^d$, we denote by $\|x\|$ the $l^2-norm$ of $x$. For the sake of simplicity, for any matrix $A$, $\|A\|$ will refer to the operator norm of $A$, specifically, $\|A\| = \sup_{u,v} \frac{<Au,v>}{\|u\|\|v\|}$. For series $\{x\}_n,\{y\}_n$, $x_n\asymp_p y_n$ is defined as: $\forall \epsilon>0$, there exists a finite $M>0$ and a finite $N>0$, such that $P(\frac{1}{M}<|\frac{y_n}{x_n}|<M)<1 - \epsilon$, $\forall n>N$. $y_n\lesssim_p x_n$ is defined as: $\forall \epsilon>0$, there exists a finite $M>0$ and a finite $N>0$, such that $P(|\frac{y_n}{x_n}|<M)<1 - \epsilon$.

\section{Main Results}

We now state our main result:

\begin{theorem}\label{thm:main}

Assume $\lambda_1<\lambda_2$, \eqref{GammaAssumption} and $E\|A_n\|^2 < \infty$, Set $g = \lambda_2-\lambda_1$. Then the Krasulina estimator $\eqref{Krasulina1}$ satisfies as $n\rightarrow \infty$,
$$
|\hat{\lambda_1}-\lambda_1|\asymp_p \frac{\|\Sigma\|}{\sqrt{n}}\cdot (\sqrt{E[\|A_n\|^2]}\bigvee\|\Sigma\|)$$ and
$$
L(V_n,\theta_1)\asymp_p \frac{\|\Sigma\|}{g\sqrt{n}}\cdot (\sqrt{E[\|A_n\|^2]}\bigvee\|\Sigma\|).
$$
\end{theorem}

In Particular, if we require the $X_k$'s to be normal random vectors, then $$\|A_n\| = \|X_n\|^2 \overset{d}{=} \sum_{j = 1} ^ d \lambda_jZ_j^2,$$ where $Z_j\overset{i.i.d.}{\sim}N(0,1)$. Consequently, we get 
$$
E[\|A_n\|^2] = E[\sum_{j=1}^d \lambda_j^2 Z_j^4 + 2\sum_{i\neq j}\lambda_i\lambda_jZ_i^2Z_j^2] = 2tr(\Sigma^2)+tr(\Sigma)^2\lesssim_p tr(\Sigma)^2.
$$

Thus we have following corollary:

\begin{corollary} Let the Assumptions of 
Theorem \ref{thm:main} be satisfied. Assume in addition that $\{X_k\}$ are i.i.d. zero mean normal random vectors with covariance matrix $\Sigma$. We have for the Krasulina scheme $\eqref{Krasulina1}$ as $n\rightarrow \infty$ that 
$$|\hat{\lambda_1}-\lambda_1|\asymp_p \frac{\|\Sigma\|tr(\Sigma)}{\sqrt{n}},
$$
and  
$$
L(V_n,\theta_1) \asymp_p \frac{\|\Sigma\|tr(\Sigma)}{g\sqrt{n}}.
$$

\end{corollary}

Note that for most of the times, we are interested in the top eigenvalue and the corresponding eigenvector. The scheme of Krasulina only computes the least eigenvalue and the corresponding eigenvector. However, our result is still useful, since by multiply with $-1$ in the original matrix, the least eigenvalue becomes the top eigenvalue. So that we can still use the scheme of Krasulina to compute the top eigenvalue with the same speed of convergence. The update formula \eqref{Krasulina1} will then become:
\begin{equation}\label{top_eig}
V_{n+1} = V_n + \gamma_{n+1}\xi_{n+1}.
\end{equation}

\section{Proof of the Theorem}
We first state a basic result in probability that will be used throughout the paper.

\begin{lemma}\label{basic_prob_thm}
Let $\{Y_n\}_n$ be a sequence of real-valued random variable. We assume that for all $n \geq1$, $Y_n$ is zero mean and square integrable. Define $S_n = \sum_{k=1}^n Y_k$. If $\sum_{n\geq1}E[Y_n^2] < \infty$, then $\{S_n\}_n$ converges to a real-valued random variable in probability.
\end{lemma}

\begin{proof}
By definition, $S_n = \sum_{k=1}^n Y_k$, since $Y_n$ is square integrable:
\begin{equation}\label{partial_sum}
E[|S_{n+r}-S_n|^2]=E[(\sum_{i=n+1}^{n+r} Y_i)^2] = \sum_{i = n+1}^{n+r}E[Y_i^2] + \sum_{n+1\leq i<j\leq n+r} 2E[Y_i\cdot Y_j].  
\end{equation}
Since $Y_n$ is zero mean, then for $i<j$:
\begin{equation*}
E[Y_i\cdot Y_j] = E[E[Y_i\cdot Y_j|\mathcal{F}_i]]=E[Y_i\cdot E[Y_j|\mathcal{F}_i]]=0,
\end{equation*}
plug it into \eqref{partial_sum}, we obtain:
$$E[|S_{n+r}-S_n|^2] = \sum_{i = n+1}^{n+r}E[Y_i^2]\leq \sum_{i>n}E[Y_i^2],$$ this is the remainder term of a convergence series, thus $\{S_n\}_n$ is Cauchy, so $\{S_n\}_n$ converges to a real-valued random variable in $\mathcal{L}^2$. By Kolmogorov inequality, Lemma \ref{basic_prob_thm} follows.
\end{proof}

Now, we start by bounding the asymptotic expectation of $\|V_n\|^2$:

\begin{lemma}
$\lim_{n\to\infty} E\|V_n\|^2<\infty$.
\end{lemma}

\begin{proof}

First, we prove that $V_n$ and $\xi_{n+1}$ are orthogonal for any $n\geq 1$. 

Let $W_n = X_{n+1}-\frac{<X_{n+1},V_n>}{\|V_n\|^2}\cdot V_n$, we have:
\begin{eqnarray*}
\xi_{n+1} &=& <X_{n+1},V_n>\cdot X_{n+1} - \frac{<X_{n+1},V_n>^2}{\|V_n\|^2} \cdot V_n\\
          &=& <X_{n+1},V_n> (X_{n+1}-\frac{<X_{n+1},V_n>}{\|V_n\|^2}\cdot V_n)\\
          &=& <X_{n+1},V_n>\cdot W_n.
\end{eqnarray*}

We note that $<W_n,V_n> = 0$, so $$\|\xi_{n+1}\| = <X_{n+1},V_n>\cdot \|W_n\| \leq <X_{n+1},V_n>\cdot\|X_{n+1}\|\leq\|X_{n+1}\|^2\|V_n\|,$$ thus: 
\begin{equation}\label{xi_bound}
    E[\|\xi_{n+1}\| |\mathcal{F}_n]\leq E[\|X_{n+1}\|^2]\cdot \|V_n\| = tr(\Sigma)\|V_n\|.
\end{equation}

Now since $\xi_n\perp V_{n-1}$, we have $$\|V_n\|^2 = \|V_{n-1}-\gamma_n\xi_n\|^2 = \|V_{n-1}\|^2 + \gamma_n^2\|\xi_n\|^2,$$ thus:
\begin{eqnarray*}
E[\|V_{n}\|^2|\mathcal{F}_{n-1}] &=& \|V_{n-1}\|^2 + \gamma_{n}^2E[\|\xi_{n}\|^2|\mathcal{F}_{n-1}]\\
                                 &\leq& \|V_{n-1}\|^2 + \gamma_{n}^2 tr(\Sigma)^2\|V_{n-1}\|^2\\
                                 &=& (1+\gamma_{n}^2tr(\Sigma)^2)\|V_{n-1}\|^2
\end{eqnarray*}

Thus:
\begin{eqnarray*}
E\|V_n\|^2  &\leq& (1+\gamma_{n}^2tr(\Sigma)^2) E\|V_{n-1}\|^2\\
            &\leq& ...\leq \prod_{i=2}^{n}(1+\gamma_i^2tr(\Sigma)^2)\cdot E\|V_1\|^2
\end{eqnarray*}

By assumption \eqref{GammaAssumption}, we have $\sum_{i=1}^\infty \gamma_i^2tr(\Sigma)^2<\infty$, thus: $\prod_{i=1}^{n-1}(1+\gamma_i^2tr(\Sigma)^2)<\infty$, thus $\lim_{n\to\infty} E\|V_n\|^2 < \infty.$
\end{proof}

Next, let $\mu(V_n) = \frac{<\Sigma V_n,V_n>}{\|V_n\|^2}$, and $a_1^{(n)} = <V_n,\theta_1>$. We first prove the convergence in probability of the sequence of $V_n$ and $a_1^{(n)}$. Specifically, $\mu(V_n)$ converges to $\lambda_1$, and $V_n$ converges to a vector which is alined with $\theta_1$. To prove that, we can recursively properly apply the inequality, to show the Cauchy property of sequence $\mu(V_n)$ and $a_1^{(n)}$.

\begin{lemma}\label{conv_mu_Vn}
$\mu (V_n) = \frac{<\Sigma V_n,V_n>}{\|V_n\|^2}$ converges a.s. to $\mu$ as $n\to\infty$.
\end{lemma}

\begin{proof}
\begin{eqnarray*}
\mu(V_{n+1}) &=& \frac{<\Sigma V_n-\gamma_{n+1}\cdot \Sigma\xi_{n+1},V_n-\gamma_{n+1}\xi_{n+1}>}{\|V_n-\gamma_{n+1}\xi_{n+1}\|^2}\\
             &=& \frac{<\Sigma V_n,V_n>+\gamma_{n+1}^2<\Sigma\xi_{n+1},\xi_{n+1}>-2\gamma_{n+1}<\xi_{n+1},\Sigma V_n>}{\|V_n\|^2 + \gamma_{n+1}^2\|\xi_{n+1}\|^2}\\
             &=& \frac{1}{1+\gamma_{n+1}^2\frac{\|\xi_{n+1}\|^2}{\|V_n\|^2}}(\mu(V_n)-2\gamma_{n+1}\frac{<\xi_{n+1},\Sigma V_n>}{\|V_n\|^2}\\
             & &+\gamma_{n+1}^2\frac{<\Sigma\xi_{n+1},\xi_{n+1}>}{\|V_n\|^2})
\end{eqnarray*}

Since:
\begin{eqnarray*}
<\xi_{n+1},\Sigma V_n> &=& <A_{n+1} V_n,\Sigma V_n>-\frac{<A_{n+1}V_n,V_n><\Sigma V_n,V_n>}{\|V_n\|^2}\\
                 &=& \|\Sigma V_n\|^2 - \frac{<\Sigma V_n,V_n>^2}{\|V_n\|^2} + <A_{n+1}V_n,\Sigma V_n> - \|\Sigma V_n\|^2 \\
                 & & -\frac{<A_{n+1}V_n,V_n><\Sigma V_n,V_n>}{\|V_n\|^2} + \frac{<\Sigma V_n,V_n>^2}{\|V_n\|^2}\\
                 &=& (<(A_{n+1}-\Sigma)V_n,\Sigma V_n>-\frac{<(A_{n+1}-\Sigma)V_n,V_n>}{\|V_n\|^2}\\
                 & &\cdot<\Sigma V_n,V_n>) + (\|\Sigma V_n\|^2 - \frac{<\Sigma V_n,V_n>^2}{\|V_n\|^2})
\end{eqnarray*}

Let 
\begin{equation}\label{fvn}
f(V_n) = \frac{\|\Sigma V_n\|^2}{\|V_n\|^2} - \frac{<\Sigma V_n,V_n>^2}{\|V_n\|^4},
\end{equation}
\begin{equation}\label{zn}
Z_n = \frac{<(A_{n+1}-\Sigma)V_n,\Sigma V_n>}{\|V_n\|^2}-\frac{<(A_{n+1}-\Sigma)V_n,V_n>}{\|V_n\|^4}\cdot<\Sigma V_n,V_n>,
\end{equation}
thus: $\frac{<\xi_{n+1},\Sigma V_n>}{\|V_n\|^2} = f(V_n) + Z_n.$

so $\mu(V_{n+1}) = \frac{1}{1+\gamma_{n+1}^2\frac{\|\xi_{n+1}\|^2}{\|V_n\|^2}}(\mu(V_n)-2\gamma_{n+1}f(V_n)-2\gamma_{n+1}Z_n+\gamma_{n+1}^2\frac{<\Sigma\xi_{n+1},\xi_{n+1}>}{\|V_n\|^2})$. 

Let 
\begin{equation}\label{abc}
a_n = \gamma_{n+1}Z_n\text{, } b_n = \gamma_{n+1}^2\frac{<\Sigma \xi_{n+1},\xi_{n+1}>}{\|V_n\|^2}\text{, } c_n = \frac{1}{1+\gamma_{n+1}^2\frac{\|\xi_{n+1}\|^2}{\|V_n\|^2}}\text{, }
\end{equation}
thus:
$$\mu(V_{n+1}) = c_n\cdot (\mu(V_n)-2\gamma_{n+1}f(V_n)-2a_n + b_n).$$

Now we have:
\begin{equation}\label{mun}
\mu(V_{n+1}) - c_n\cdot \mu(V_n) = -2\gamma_{n+1}c_nf(V_n)-2a_nc_n + b_nc_n.
\end{equation}

For series $\{a_n\}$, since $Z_n$ is centered and $E[Z_n^2]$ is bounded, by lemma \ref{basic_prob_thm}: $$\sum_{i>k} Var(a_i)\asymp_p \sum_{i>k} \gamma_i^2 <\infty,$$ thus $\sum_{n=1}^\infty a_n<\infty$.

For series $\{b_n\}$, by \eqref{xi_bound}: $$E[\|\xi_n\||\mathcal{F}_{n-1}]\leq tr(\Sigma)\|V_n\|,$$
thus 
$$E[b_n|\mathcal{F}_{n}]=\gamma_{n+1}^2E[\frac{<\Sigma\xi_{n+1},\xi_{n+1}>}{\|V_n\|^2}|\mathcal{F}_n]\leq \gamma_{n+1}^2\|\Sigma\|tr(\Sigma)^2.$$
By \eqref{xi_bound}, we have $\sum_{n=1}^\infty b_n<\infty$.

For series $\{c_n\}$, $\prod c_n = \prod \frac{1}{1+\gamma_{n+1}^2\frac{\|\xi_{n+1}\|^2}{\|V_n\|^2}}$ converges when $\prod {1+\gamma_{n+1}^2\frac{\|\xi_{n+1}\|^2}{\|V_n\|^2}}$ converges. $\prod {1+\gamma_{n+1}^2\frac{\|\xi_{n+1}\|^2}{\|V_n\|^2}}$ has the same convergence properties as $\sum \gamma_{n+1}^2\frac{\|\xi_{n+1}\|^2}{\|V_n\|^2}$. By \eqref{xi_bound}, $$E[\frac{\|\xi_{n+1}\|^2}{\|V_n\|^2}|\mathcal{F}_n]\leq tr(\Sigma)^2,$$ we have $\prod_{n=1}^\infty c_n<\infty$.

And by Cauchy-Schwartz inequality:
\begin{equation}\label{Cauchyschwarz}
f(V_n) = \frac{\|\Sigma V_n\|^2}{\|V_n\|^2} - \frac{<\Sigma V_n,V_n>^2}{\|V_n\|^4}\geq0.
\end{equation}

Now, if $\liminf \mu(V_n) < \limsup \mu(V_n)$, choose $a,b$ such that $\liminf \mu(V_n) <a<b< \limsup \mu(V_n)$, find $m_1,n_1$ large enough, such that $\mu(V_{n_1})<a,\mu(V_{m_1})>b,$ and for all $n_1<j<m_1$, we have $a\leq \mu(V_j)\leq b$. Thus:
$$\mu(V_{m_1})-\mu(V_{n_1})\prod_{i=n_1}^{m_1-1}c_i>b-a.$$

On the other hand:
\begin{eqnarray}\label{muiter}
\mu(V_{m_1})-\mu(V_{n_1})\prod_{i=n_1}^{m_1-1}c_i &=& \sum_{j=n_1}^{m_1-1} [(-2\gamma_{j+1}\cdot f(V_j)-2a_j+b_j)\cdot\prod_{i=j}^{m_1-1}c_j]\\
&\leq& \sum_{j=n_1}^{m_1-1} [(-2a_j+b_j)\cdot\prod_{i=j}^{m_1-1}c_j]\nonumber\\
&\to& 0\text{ as }n_1,m_1\to\infty\nonumber,
\end{eqnarray}
which is a contradiction, thus $\mu(V_n)\to\mu$ with probability 1.
\end{proof}

\begin{lemma}\label{conv_a_1}
 $a_1^{(n)} = <V_n,\theta_1>$, where $\theta_1$ is the eigenvector of $\lambda_1$, $a_1^{(n)}$ converges to some value $a_1$ with probability $1$ as $n\to\infty$.
\end{lemma}

\begin{proof}
Since $V_{n+1} = V_n - \gamma_{n+1}\xi_{n+1}$, $\xi_{n+1} = A_{n+1}V_n-\frac{<A_{n+1}V_n,V_n>}{\|V_n\|^2}V_n$, by definition of $a_1^{(n)} = <V_n,\theta_1>$ and $\mu (V_n) = \frac{<\Sigma V_n,V_n>}{\|V_n\|^2}$, also by the nature: $<\Sigma V_n,\theta_1> = <V_n,\Sigma \theta_1> = <V_n,\lambda_1\theta_1> = \lambda_1a_1^{(n)}$, we have:
\begin{eqnarray*}
a_1^{(n+1)}&=&<V_{n+1},\theta_1>\ =\ <V_n - \gamma_{n+1}\xi_{n+1},\theta_1>\\
& = & <V_n,\theta_1> - \gamma_{n+1} <A_{n+1}V_n-\frac{<A_{n+1}V_n,V_n>}{\|V_n\|^2}V_n,\theta_1>\\
& = & a_1^{(n)} + \gamma_{n+1} <\frac{<\Sigma V_n,V_n>}{\|V_n\|^2}V_n + \frac{<(A_{n+1}-\Sigma)V_n,V_n>}{\|V_n\|^2}V_n-\Sigma V_n\\
& & +(\Sigma-A_{n+1})V_n,\theta_1>\\
& = & a_1^{(n)} + \gamma_{n+1}(\mu(V_n)-\lambda_1)a_1^{(n)} + \gamma_{n+1} Z_n'\\
& = & a_1^{(n)}(1 + \gamma_{n+1}(\mu(V_n)-\lambda_1)) + \gamma_{n+1} Z_n',
\end{eqnarray*}
where $Z_n' = <(\Sigma-A_{n+1})V_n,\theta_1> + \frac{<(A_{n+1}-\Sigma)V_n,V_n>}{\|V_n\|^2}a_1^{(n)}$.

Since $E[\|V_n\|^2] = E[\|V_{n-1}\|^2]+\gamma_n^2E[\|\xi_n\|^2]\leq E[\|V_{n_1}\|^2]+\gamma_n^2 \|\Sigma\|^2 E\|V_{n-1}\|^2\leq\prod_{n=1}^{\infty}(1+\gamma_n^2\|\Sigma\|^2)\leq \infty$, $Z_n'$ is centered and $E[Z_n'^2]$ is bounded, by lemma \ref{basic_prob_thm}, $\sum_{n=1}^\infty\gamma_nZ_n'<\infty.$

Now, if $\liminf a_1^{(n)} < \limsup a_1^{(n)}$, choose $a,b$ such that $\liminf a_1^{(n)} < a<b < \limsup a_1^{(n)}$, find $m_1,n_1$, such that: $m_1\geq n_1\geq N$, $a_1^{(m_1)}<a$, $a_1^{(n_1)}>b$, for $j\in (n_1,m_1)$, $a\leq a_1^{(j)}\leq b$. Since $\lambda_1$ is the smallest eigenvalue, $\mu(V_k)\geq\lambda_1$.

Thus: $$a_1^{(m_1)}-a_1^{(n_1)}\prod_{k=n_1}^{m_1} (1+\gamma_{k+1}(\mu(V_k)-\lambda_1))\leq a_1^{(m_1)}-a_1^{(n_1)} < a-b\leq0.$$

On the other hand:

\begin{eqnarray}\label{aidif}
& &a_1^{(m_1)}-a_1^{(n_1)}\prod_{k=n_1}^{m_1} (1+\gamma_{k+1}(\mu(V_k)-\lambda_1)) \nonumber\\
&=& \sum_{j=n_1}^{m_1-1}\gamma_j Z_j'\prod_{i=j}^{m-1}({1+\gamma_{j+1}(\mu(V_j)-\lambda_1)})\nonumber\\
&\geq& \sum_{j=n_1}^{m_1-1}\gamma_j Z_j'
\end{eqnarray}

Since $\sum_{j=1}^\infty \gamma_jZ_j'<\infty$, let $n_1\to\infty$, we can let $\sum_{j=n_1}^{m_1-1}\gamma_j Z_j'$ as closed to 0 as we want, which is a contradiction.

Thus $a_1^{(n)}\to a_1$ with probability 1.
\end{proof}

Now we get the idea that $\mu(V_n)$ and $a_1^{(n)}$ are both convergence with probability 1, and by the proof above, all coefficients in \eqref{mun} are convergence with probability 1, so does the part $\gamma_{n+1}c_nf(V_n)$. By find the convergence rate for each of these parts, we can find the convergence rate for $\mu(V_n)$.

\begin{lemma}\label{conv_rate_mu_vn}
 (1) $\mu(V_n)\to\lambda_1$ as $n\to\infty$ with probability $1$, and (2) the convergence rate of $\frac{<A_nV_n,V_n>}{\|V_n\|^2}$ to $\lambda_1$ is in the order of $O(\frac{\|\Sigma\|}{\sqrt{n}}\cdot (\sqrt{E[\|A_n\|^2]})\bigvee\|\Sigma\|)$.
\end{lemma}

\begin{proof}
(1) 
\begin{eqnarray*}
a_1^{(n+1)} & = & <V_{n+1},\theta_1>\ =\ <V_{n+1},\theta_1>\ =\ <V_n-\gamma_{n+1}\xi_{n+1},\theta_1>\\
& = &<V_n,\theta_1> - \gamma_{n+1}<A_{n+1}V_n - \frac{<A_{n+1}V_n,V_n>}{\|V_n\|^2},\theta_1>\\
& = & a_1^{(n)} + \gamma_{n+1}\frac{<A_{n+1}V_n,V_n>}{\|V_n\|^2}a_1^{(n)} - \gamma_{n+1}<A_{n+1}V_n,\theta_1>\\
& = & a_1^{(n)} + \gamma_{n+1}\frac{<\Sigma V_{n},V_n>}{\|V_n\|^2}a_1^{(n)} - \gamma_{n+1} <V_n,\Sigma\theta_1>\\
&   & + \gamma_{n+1}\frac{<A_{n+1}V_n,V_n>}{\|V_n\|^2}a_1^{(n)} - \gamma_{n+1}<A_{n+1}V_n,\theta_1>\\
&   & - \gamma_{n+1}\frac{<\Sigma V_n,V_n>}{\|V_n\|^2}a_1^{(n)} + \gamma_{n+1} <V_n,\Sigma \theta_1>\\
& = & a_1^{(n)} (1 + \gamma_{n+1}(\mu(V_n)-\lambda_1)) + \gamma_{n+1}Z_n',
\end{eqnarray*}

where $Z_n' = <(\Sigma-A_{n+1})V_n,\theta_1>+\frac{<(A_{n+1}-\Sigma)V_n,V_n>}{\|V_n\|^2}a_1^{(n)}$, which is centered and bounded, then by Jensen's inequality:
\begin{eqnarray*}
E|a_1^{(n+1)}| & \geq & E|a_1^{(n)}|(1+\gamma_{n+1}(\frac{E[\mu(V_n)|a_1^{(n)}|]}{E|a_1^{(n)}|}-\lambda_1))\\
& \geq & \prod_{k=1}^n (1+\gamma_{k+1}(\frac{E[\mu(V_k)|a_1^{(k)}|]}{E|a_1^{(k)}|}-\lambda_1)) E|a_1^{(1)}|
\end{eqnarray*}
By Lemma \ref{conv_a_1}, $\{a_1^{(n)}\}$ convergence, then $$ \prod_{k=1}^\infty (1+\gamma_{k+1}(\frac{E[\mu(V_k)|a_1^{(k)}|]}{E|a_1^{(k)}|}-\lambda_1))<\infty,$$ thus: $$\sum_{k=1}^\infty \gamma_{k+1}(\frac{E[\mu(V_k)|a_1^{(k)}|]}{E|a_1^{(k)}|}-\lambda_1)<\infty.$$ By \eqref{GammaAssumption}, $\lim_{k\to\infty}\frac{E[\mu(V_k)|a_1^{(k)}|]}{E|a_1^{(k)}|}-\lambda_1 = 0$.

By dominant convergence theorem: $\lim_{k\to\infty} a_1^{(k)} = a_1$, $\lim_{k\to\infty} \mu(V_k) = \mu$. Thus: $\frac{\mu a_1}{a_1} = \lambda_1$, therefore, $\mu = \lambda_1$.

(2)
\begin{eqnarray*}
\lambda_1-\frac{<A_nV_n,V_n>}{\|V_n\|^2} &=& (\lambda_1 - \mu(V_n)) + (\mu(V_n)-\frac{<A_nV_n,V_n>}{\|V_n\|^2})\\
& = & (\lambda_1 - \mu(V_n)) + (\frac{<(\Sigma-A_n)V_n,V_n>}{\|V_n\|^2})
\end{eqnarray*}

Since $E[\frac{<(\Sigma-A_n)V_n,V_n>}{\|V_n\|^2}]=0$, we only need to consider $|\lambda_1-\mu(V_n)|$. From \eqref{mun} we have: $$\mu(V_{n+1}) - c_n\cdot \mu(V_n) = -2\gamma_{n+1}c_nf(V_n)-2a_nc_n + b_nc_n = (-2\gamma_{n+1}f(V_n)-2a_n + b_n)c_n,$$ where $a_j$, $b_j$ and $c_j$ are defined the same as $\eqref{abc}$. The same way as we get \eqref{muiter}, keep increase $V_{n+1}$ to $V_m$ recursively, we have:
$$\mu(V_{m}) - \mu(V_n)\prod_{i=n}^{m-1}c_i = \sum_{j=n}^{m-1}(b_j-2\gamma_{j+1}f(V_j)-2a_j)\prod_{i=j}^{m-1}c_i.$$
Now, by \eqref{xi_bound}: $E[\|\xi_n\||\mathcal{F}_{n-1}]\leq tr(\Sigma)\|V_n\|.$

For $b_j$ part, 
\begin{eqnarray*}
\sum_{j = n}^\infty E[b_j|\mathcal{F}_j] &=& \sum_{j = n}^\infty \gamma_{j+1}^2 E[\frac{<\Sigma\xi_{j+1},\xi_{j+1}>}{\|V_j\|^2}|\mathcal{F}_j]\leq\sum_{j = n}^\infty \gamma_{j+1}^2 \frac{\|\Sigma\|E[\|\xi_{j+1}\|^2|\mathcal{F}_j]}{\|V_j\|^2}\\
&\leq& \sum_{j = n}^\infty \gamma_{j+1}^2 \frac{\|\Sigma\|tr(\Sigma)^2\|V_j\|^2}{\|V_j\|^2}=\sum_{j = n}^\infty \gamma_{j+1}^2\cdot c,
\end{eqnarray*}
thus its rate of convergence is $O(\frac{1}{n})$

For $a_j$ part, $\sum_{j=n}^\infty a_j = \sum_{j=n}^\infty \gamma_{j+1}Z_j$, $Z_j$ is centered and $E[Z_j^2]$ is bounded, by lemma \ref{basic_prob_thm}, $E[|S-S_n|^2]\leq \sum_{i>n} E[a_i^2]$, whose rate of convergence is $O(\frac{1}{n})$, thus $\sum_{j=n}^\infty a_j$ has the rate of convergence $O(\frac{1}{\sqrt{n}})$.

For $c_j$ part, by proof of the lemma \ref{conv_mu_Vn}, $\prod_{i=n}^\infty c_i$ has the same convergence properties as  $\sum_{i=n}^\infty \gamma_{i+1}^2\frac{\|\xi_{i+1}\|^2}{\|V_i\|^2}$. By \eqref{xi_bound}:
$$E[\frac{\|\xi_{i+1}\|^2}{\|V_i\|^2}|\mathcal{F}_i]\leq E[\frac{tr(\Sigma)^2\|V_i\|^2}{\|V_i\|^2}] = tr(\Sigma)^2,$$ thus $\prod_{i=n}^\infty c_i$ has the rate of convergence $O(\frac{1}{n})$.

For $f(V_j)$ part, by assumption 2, rewrite $V_n = \sum_{i=1}^d a_i^{(n)}\theta_i$, where $d$ is the dimension. From $\eqref{muiter}$, we have:
$\sum_{n=1}^\infty \gamma_{n+1}f(V_n)\prod_{k=1}^{n-1}(1+\gamma_{k+1}^2\frac{\|\xi_{k+1}\|^2}{\|V_k\|^2})^{-1}<\infty$ with probability $1$.
Since we have $\gamma_n\asymp_p \frac{1}{n}$ and $f(V_n)\geq0$ $\forall n$, if $\liminf_{n\to\infty}f(V_n) = c$, then $\sum_{n=1}^\infty\gamma_{n+1}f(V_n)\prod_{k=1}^{n-1}(1+\gamma_{k+1}^2\frac{\|\xi_{k+1}\|^2}{\|V_k\|^2})^{-1}=\infty$, thus $c=0$.

Now, by nature of eigenvector and eigenvalue, as well as assumption 2: $\theta_i^2 = 1$, $\theta_i\theta_j = 0$ for $i\neq j$, and $\|V_n\|^2 = \sum_{i=1}^d (a_i^{(n)})^2$.

Thus:
\begin{eqnarray}\label{fvneigdec}
f(V_n) & = & \frac{\|\Sigma V_n\|^2}{\|V_n\|^2} - \frac{<\Sigma V_n,V_n>^2}{\|V_n\|^4}\nonumber\\
& = & \frac{(\sum_{i=1}^d a_i^{(n)}\lambda_i\theta_i)^2}{\|V_n\|^2}-\mu(V_n)^2\nonumber\\
& = & \frac{\sum_{i=1}^d (a_i^{(n)})^2(\lambda_i^2-\mu(V_n)^2)}{\|V_n\|^2},
\end{eqnarray}
which leads to the result: $f(V_j)\to 0$ with the same rate of $\mu(V_n)\to\lambda_1.$

Thus, $\frac{<A_nV_n,V_n>}{\|V_n\|^2}$ converges to $\lambda_1$ the same rate as $a_j$ part, has the rate of convergence $O(\frac{1}{\sqrt{n}})$. More precisely, by proof of the Lemma \ref{basic_prob_thm}, $E[|S_{n+r}-S_n|^2] \leq \sum_{i>n}E[X_i^2]$ if $\{X_n\}_n$ is $0$ mean. Then for $a_j = \gamma_{j+1}Z_j$, we have $$E[|S-S_n|^2]\leq \sum_{i>n} E[a_i^2]\lesssim_p \sum_{i>n} \frac{1}{i^2}E[Z_i^2].$$ 

Now for $Z_n$, by $\eqref{zn}$, we have:
\begin{eqnarray*}
\|Z_n\| &=& \|\frac{<(A_{n+1}-\Sigma)V_n,\Sigma V_n>}{\|V_n\|^2}-\frac{<(A_{n+1}-\Sigma)V_n,V_n>}{\|V_n\|^4}\cdot<\Sigma V_n,V_n>\|\\
&\leq& \|\frac{<(A_{n+1}-\Sigma)V_n,\Sigma V_n>}{\|V_n\|^2}\|+\|\frac{<(A_{n+1}-\Sigma)V_n,V_n>}{\|V_n\|^4}\cdot<\Sigma V_n,V_n>\|\\
&\leq& \|\frac{<(A_{n+1}-\Sigma)V_n,V_n>}{\|V_n\|^2}\|\cdot \|\Sigma\|+\|\frac{<(A_{n+1}-\Sigma)V_n,V_n>}{\|V_n\|^2}\\
& &\cdot\frac{<\Sigma V_n,V_n>}{\|V_n\|^2}\|\\
&\lesssim_p& \|\frac{<(A_{n+1}-\Sigma)V_n,V_n>}{\|V_n\|^2}\|\cdot \|\Sigma\|\\
&\leq& \|A_{n+1}-\Sigma\|\|\Sigma\|\\
&\leq& (\|A_{n+1}\| + \|\Sigma\|)\|\Sigma\|.
\end{eqnarray*}

Thus:
\begin{eqnarray*}
E[Z_n^2] &\leq& \|\Sigma\|^2E[\|A_{n+1}\|^2 + \|\Sigma\|^2 + 2\|A_{n+1}\|\|\Sigma\|]\\
&\lesssim_p& \|\Sigma\|^2 E[\|A_{n+1}\|^2 + \|\Sigma\|^2]\\
&\asymp_p& \|\Sigma\|^2\cdot (E[\|A_n\|^2]\bigvee \|\Sigma\|^2).
\end{eqnarray*}

So $E[|S-S_n|^2]$ has rate of convergence $O(\frac{1}{n}\cdot\|\Sigma\|^2\cdot (E[\|A_n\|^2]\bigvee \|\Sigma\|^2))$, thus $\sum_{j=n}^\infty a_j$ has rate of convergence $O(\frac{\|\Sigma\|}{\sqrt{n}}\cdot (\sqrt{E[\|A_n\|^2]}\bigvee\|\Sigma\|))$.

\end{proof}

\begin{lemma}

(1) ${V}_n\to a_1^{(n)}\theta_1$ with probability $1$ and (2) $\frac{<V_n,\theta_1>^2}{\|V_n\|^2}$ approach to $1$ in the order of $ \frac{d\|\Sigma\|}{g\sqrt{n}}\cdot (\sqrt{E[\|A_n\|^2]}\bigvee\|\Sigma\|)$ with probability $1$.
\end{lemma}

\begin{proof}

(1) We already proved that $f(V_n)\to 0$ and $\mu(V_n)\to\lambda_1$ in lemma \ref{conv_rate_mu_vn}, thus $\lambda_i-\mu(V_n)>0$ for $i\neq 1$ when $n$ large enough. By \eqref{fvneigdec}, $0 = \lim_{n\to\infty}f(V_n) = \lim_{n\to\infty}\frac{\sum_{i=1}^d (a_i^{(n)})^2(\lambda_i^2-\mu(V_n)^2)}{\|V_n\|^2}$, $a_i^{(n)} = 0$ when $i\neq 1$, thus ${V}_n\to a_1^{(n)}\theta_1$ with probability $1$.

(2) By previous argument, we have: 
\begin{eqnarray*}
f(V_n) &=& \frac{\sum_{i=1}^d (a_i^{(n)})^2(\lambda_i^2-\mu(V_n)^2)}{\|V_n\|^2}\\
&=& \frac{(a_1^{(n)})^2(\lambda_1^2-\mu(V_n)^2)}{\|V_n\|^2} + \frac{\sum_{i=2}^d (a_i^{(n)})^2(\lambda_i^2-\mu(V_n)^2)}{\|V_n\|^2},
\end{eqnarray*}
convergence with the same rate of $\mu(V_n)\to\lambda_1,$ we have $\frac{\sum_{i = 2}^\infty (a_i^{(n)})^2(\lambda_i^2-\mu(V_n)^2)}{\|V_n\|^2}\to 0$ at least the same rate as $\frac{(a_1^{(n)})^2(\lambda_1^2-\mu(V_n)^2)}{\|V_n\|^2}\to 0$.

By part (1), $\mu(V_n)$ has rate of convergence $O(\frac{\|\Sigma\|}{\sqrt{n}}\cdot (\sqrt{E[\|A_n\|^2]})\bigvee\|\Sigma\|)$, we have 
$$\frac{\sum_{i = 2}^\infty (a_i^{(n)})^2(\lambda_i^2-\lambda_1^2)}{\|V_n\|^2} \asymp_p \frac{\|\Sigma\|}{\sqrt{n}}\cdot (\sqrt{E[\|A_n\|^2]}\bigvee\|\Sigma\|)\cdot\frac{(a_1^{(n)})^2 \lambda_1}{\|V_n\|^2},$$
let $g = |\lambda_1-\lambda_2|$, thus:
\begin{eqnarray*}
\sum_{i = 2}^\infty(a_i^{(n)})^2 &\asymp_p& \frac{\|\Sigma\|}{\sqrt{n}}\cdot (\sqrt{E[\|A_n\|^2]}\bigvee\|\Sigma\|)\cdot\frac{(a_1^{(n)})^2 \lambda_1}{|(\lambda_i-\lambda_1)(\lambda_i+\lambda_1)|}\\
&\lesssim_p& \frac{\|\Sigma\|\|V_n\|^2}{g\sqrt{n}}\cdot (\sqrt{E[\|A_n\|^2]}\bigvee\|\Sigma\|)
\end{eqnarray*}

Now by assumption 2, $\|V_n\|^2 = \sum_{i=1}^d (a_i^{(n)})^2$, thus:
$$\|V_n\|^2 - (a_1^{(n)})^2 = \sum_{i = 2}^\infty(a_i^{(n)})^2 \lesssim_p \frac{\|\Sigma\|\|V_n\|^2}{g\sqrt{n}}\cdot (\sqrt{E[\|A_n\|^2]}\bigvee\|\Sigma\|).$$ Above all:
$$1 - \frac{<V_n,\theta_1>^2}{\|V_n\|^2} \lesssim_p \frac{\|\Sigma\|}{g\sqrt{n}}\cdot (\sqrt{E[\|A_n\|^2]}\bigvee\|\Sigma\|).$$

\end{proof}

\section{Experiment}
The dataset $X \in \mathbb{R}^{10^6\times10}$ was just generated through its singular value decomposition. Specifically, we fix a $10\times10$ diagonal matrix $\Sigma = diag\{1, 0.9, \cdots, 0.9\}$ and generate random orthogonal projection matrix $U \in \mathbb{R}^{10^6\times10}$ and random orthogonal matrix $V \in \mathbb{R}^{10\times10}$. And the dataset $X = \sqrt{n}U\Sigma V^T$, which guarantees that the matrix $A = \frac{1}{n}X^TX$ has eigen-gap 0.1. See Figure.1.

\begin{figure}[!hbtp]
\begin{center}
  \includegraphics[width=0.72\textwidth]{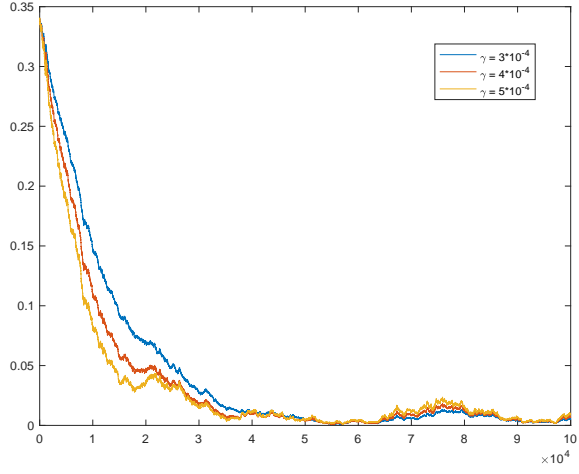}
  \caption{Convergence of Krasulina Scheme}
  \label{difference}
 \end{center}
\end{figure}

\section{Conclusion}

We derived the asymptotic rate of convergence for the estimation of the smallest eigenvalue and corresponding eigenvector of the Krasulina scheme. There are several important questions related to Online PCA.

\begin{enumerate}

\item  The Krasulina scheme only requires $O(d)$ storage space complexity against $O(d^2)$ for standard PCA in the offline setting, however, we paid a price in the rate of convergence that is significantly slower than offline setting. See Table.\ref{tab_schemes}.

\begin{table}
\newcommand{\tabincell}[2]{\begin{tabular}{@{}#1@{}}#2\end{tabular}}
\centering
\caption{Comparison of different schemes with Gaussian assumption. The convergence rates are given for the operator norm. For the sparse PCA scheme of \cite{cai_sparse_2013}, $k_q^*$ denotes the sparsity level of the eigenvectors.} \label{tab_schemes}
\begin{tabular}{|c|c|c|c|}
\hline
Scheme&\tabincell{c}{Space\\ complexity}&\tabincell{c}{Convergence rate}&Setting\\
\hline
Standard PCA&$O(nd^2$)&$O(\|\Sigma\|\cdot (\sqrt{\frac{r(\Sigma)}{n}}\bigvee\frac{r(\Sigma)}{n}))$&Offline\\
\hline
Sparse PCA \cite{cai_sparse_2013}&$O(nd^2$)&$O(\frac{k_q^*}{n\lambda}(d+\log{\frac{d}{k_q^*}}))$&Offline\\
\hline
Krasulina&$O(d)$&$O(\frac{\|\Sigma\|tr(\Sigma)}{\sqrt{n}})$ & Online\\ 
\hline
\end{tabular}
\end{table}

An interesting question would be whether the Krasulina scheme can achieve the offline rate of convergence. 

The simulation study seems to confirm the slow convergence rate of Krasulina's scheme. It would be interesting to build an acceleration for this scheme. This problem has been investigated by \cite{de_sa_accelerated_2017} where negative numerical results were provided for usual acceleration schemes. Therefore this question remains largely open.

\item Note that the proof argument in the original paper \cite{krasulina_method_1969} only gives the consistency of the smallest eigenvalue and corresponding eigenvector for the Krasulina scheme. As we built upon this argument in this paper, we only provide the rate of convergence for the smallest eigenvalue and corresponding eigenvector. We can extend the result to the top eigenvalue and cooresponding eigenvector by \eqref{top_eig}, however, tackling other eigenvalues will require a new argument.

\item The convergence rates of the Krasulina estimator depends on the multiplier $c$ in the learning rate $\gamma_n$ when we take $\gamma_n = \frac{c}{n}$. If it is too low, the rate of convergence will be slower than $O(\frac{1}{\sqrt{n}})$, if it is too high, the constant in the rate of convergence will be large. Is there a simple and practical way to choose $c$?

 \item Finally, it would be of interest to derive rates of convergence for other online PCA schemes including Oja and naive PCA.

\end{enumerate}

\nocite{*}


\end{document}